\documentclass[journal,twoside,web]{ieeecolor}

\usepackage{graphics} 
\usepackage{epsfig} 
\usepackage{mathptmx} 
\usepackage{amsmath} 
\usepackage{amssymb}  
\usepackage{xcolor}
\usepackage{mathtools}

\usepackage[normalem]{ulem}

\usepackage{lcsys}
\usepackage{graphicx}
\newtheorem{theorem}{Theorem}[section]

\newtheorem{lemma}[theorem]{Lemma}

\newtheorem{definition}[theorem]{Definition}
\newtheorem{assumption}[theorem]{Assumption}
\newtheorem{remark}[theorem]{Remark}

\usepackage{algorithm}
\usepackage{algorithmic}
\usepackage[hidelinks]{hyperref}
\usepackage{cite}

\newif\ifprerev
\prerevfalse

\newif\ifpostrev
\postrevtrue

\newcommand{\M}{\mathcal{M}}
\newcommand{\X}{\mathcal{X}}
\newcommand{\Exp}{\mathrm{Exp}}
\newcommand{\Log}{\mathrm{Log}}

\newcommand{\Reg}{\mathrm{Reg}}
\newcommand{\dist}{d}
\newcommand{\Bus}{B}

\DeclareMathOperator*{\argmin}{arg\,min}
\newcommand{\sigmatwo}{\sigma_2(W)}

\begin{document}

\title{Curvature-Independent Regret Bounds for Distributed Online Optimization on Hadamard Manifolds}
\author{Zhanyuan Cai, Emre Sahinoglu and Shahin Shahrampour
\thanks{This work is supported in part by NSF Award ECCS-2240788 as well as NSF CAREER Award ECCS-2442321.}
\thanks{The authors are with the Department of Mechanical and Industrial Engineering,
         Northeastern University, Boston, MA 02115.
         {\tt\small \ Emails: \{cai.zha; sahinoglu.m;s.shahrampour\}@northeastern.edu}.}
 }

\pagestyle{empty}
\maketitle
\thispagestyle{empty}

\begin{abstract}
 This work addresses decentralized online Riemannian optimization on Hadamard manifolds. Prior work under geodesic convexity (g-convexity) may require curvature information in the optimization analysis, typically through a finite lower bound on the sectional curvature. Curvature may also enter the step size or contraction factor of tangent-space Riemannian consensus schemes. In this work, we relax the curvature dependence for a narrower class of horospherical convex (h-convex) functions. We study Distributed Riemannian Online Gradient Descent (D-ROGD), which combines local Riemannian h-subgradient updates with an implicit Fr\'echet-mean consensus. For h-convex and strongly h-convex local objectives, we establish $O(\sqrt{T})$ and $O(\log T)$ static regret, respectively, matching the corresponding Euclidean rates with respect to $T$, with network dependence governed solely by the spectral gap. To our knowledge, these are the first curvature-independent regret guarantees for decentralized online optimization on Hadamard manifolds. Experiments on hyperbolic embeddings corroborate the predicted rates, with no observable degradation due to curvature.
\end{abstract}
 
 
\section{Introduction}

Many control and engineering problems naturally involve decision variables constrained to nonlinear geometric spaces. Representative examples include covariance estimation and averaging over symmetric positive-definite matrices \cite{moakher2005differential}, Riemannian system identification with structured positive-definite variables \cite{sato2019riemannian1}, and feedback-policy synthesis over submanifolds induced by control constraints \cite{talebi2023policy}. Riemannian optimization provides a framework for such problems, enabling optimization directly on the underlying manifold while respecting the geometry of the decision space \cite{absil2009optimization,boumal2023introduction}.

On the other hand, the sequential and distributed nature of many control, estimation, and learning systems motivates online and decentralized extensions of Riemannian optimization. In the online setting, the objective varies over time and is revealed only after a decision is made \cite{wang2023online,wang2025online}, while in the decentralized setup, data and/or computation are distributed across a network of agents \cite{shahrampour2017distributed}. Combining these two settings, we consider agents that sequentially update their (manifold-valued) decisions, observe private local losses, and communicate only with neighboring agents, with the goal of achieving a small regret, defined relative to the best fixed collective decision in hindsight. While this framework is well understood in the Euclidean domain, its Riemannian counterpart introduces additional geometric challenges.

Unlike in the Euclidean setting, curvature can affect decentralized online Riemannian optimization through two distinct mechanisms. On the optimization side, standard analyses under g-convexity rely on a Riemannian cosine inequality containing the factor $\sqrt{|\kappa|}D\coth(\sqrt{|\kappa|}D)$, which requires a finite sectional-curvature lower bound and deteriorates with both curvature $\kappa$ and domain diameter $D$ \cite{zhang2016first}. On the network side, curvature-aware tangent-space consensus schemes may require curvature-dependent step sizes and contraction factors \cite{chen2024decentralized,sahinoglu2025decentralizedonlineriemannianoptimization,cai2026decentralized}. Implicit Fr\'echet-mean consensus avoids this dependence on Hadamard manifolds \cite{chen2024decentralized}, but curvature-independent error guarantees compatible with the diminishing step sizes required for $O(\log T)$ regret have not been established. Thus, curvature-independent regret bounds require controlling both optimization progress and network error without curvature-dependent terms.

We address these two sources of curvature dependence through complementary geometric mechanisms. On the optimization side, we adopt horospherical convexity (h-convexity), which provides curvature-free inequalities through Busemann functions \cite{criscitiello2025horospherically,sahinoglu2025online}. On the network side, we employ implicit weighted Fr\'echet-mean consensus, whose contraction depends only on network connectivity and requires no curvature-dependent consensus step size \cite{chen2024decentralized}. We then develop a curvature-independent network-error analysis for arbitrary step size schedules, accommodating both the h-convex and strongly h-convex online regimes without requiring a finite lower bound on the sectional curvature. This curvature independence comes at the cost of a more restrictive objective class and a consensus step that requires solving a Fr\'echet-mean subproblem in each round. Our main contributions are as follows:

\begin{itemize}
    \item \textbf{Curvature-independent regret.} We propose Distributed Riemannian Online Gradient Descent (D-ROGD) and establish $O(\sqrt{T})$ and $O(\log T)$ static regret for h-convex and strongly h-convex losses, respectively (Theorems~\ref{thm:hconvex} and~\ref{thm:stronghconvex}), without requiring a finite lower bound on the sectional curvature. The rates match their Euclidean counterparts in $T$, with network dependence governed by the spectral gap \cite{zinkevich2003online,hazan2007logarithmic}. To our knowledge, these are the first curvature-free regret guarantees for decentralized online optimization on Hadamard manifolds.
    \item \textbf{Network error under arbitrary step sizes.} We establish a curvature-independent bound on the network error induced by implicit weighted Fr\'echet-mean consensus under arbitrary step size schedules (Lemma~\ref{lem:neterror}), with network dependence governed solely by the spectral gap. This bound remains valid for diminishing step sizes, including $\eta_t=1/(\mu t)$ in the strongly h-convex analysis. While this result is essential for our regret bound, it may also be of independent interest.
    \item \textbf{Numerical validation.} We evaluate D-ROGD on hyperbolic embeddings of real hierarchical data while varying both manifold curvature and network connectivity. The experiments exhibit regret growth consistent with the theoretical rates, show no systematic deterioration as curvature increases, and confirm that poorer network connectivity leads to higher regret, as predicted by the spectral-gap dependence of the theoretical bounds.
\end{itemize}

\subsection{Literature Review}

\noindent \textbf{Riemannian Online Optimization and Curvature Independence}. Online optimization over Riemannian manifolds has been developed primarily under  g-convexity, extending classical regret analysis to curved decision spaces \cite{wang2023online,wang2025online}. Subsequent works have considered zeroth-order feedback and tracking \cite{maass2022tracking}, projection-free online optimization \cite{hu2023riemannian}, and optimistic methods with dynamic regret \cite{wang2025riemannian}. A persistent feature of this literature is that standard g-convex analyses rely on geometric comparison inequalities whose constants depend on a finite lower bound on the sectional curvature \cite{zhang2016first}, motivating recent efforts toward curvature-independent online guarantees. One route, proposed by Roux et al. \cite{roux2025implicit}, retains g-convexity and removes geometric constants through inexact implicit updates, at the price of replacing standard explicit first-order steps with approximately solved implicit subproblems. A complementary structural route is provided by h-convexity, which uses Busemann functions to obtain curvature-free comparison inequalities on Hadamard manifolds \cite{criscitiello2025horospherically,lewis2024horoballs}. Building on this framework, Sahinoglu and Shahrampour \cite{sahinoglu2025online} establish curvature-independent $O(\sqrt{T})$ and $O(\log T)$ regret for online optimization with h-convex and strongly h-convex objectives, respectively. These results, however, concern a single online learner and leave open the question of how to preserve curvature independence in the decentralized setting, where network disagreement introduces an additional geometric challenge.

\noindent\textbf{Decentralized Riemannian Optimization and Consensus.}
Distributed online Riemannian optimization has so far been developed primarily under g-convexity. Chen and Sun \cite{chen2024decentralized} analyze dynamic regret on Hadamard manifolds, while Sahinoglu and Shahrampour \cite{sahinoglu2025decentralizedonlineriemannianoptimization} extend distributed online optimization beyond Hadamard manifolds to settings with zeroth-order feedback. More recently, Cai et al. \cite{cai2026decentralized} establish $O(\log T)$ regret for strongly g-convex objectives using a network-error analysis that accommodates diminishing step sizes. Despite these advances, curvature-dependent quantities remain either in the optimization analysis under g-convexity or in the consensus mechanism. In contrast, we combine h-convexity with implicit weighted Fr\'echet-mean consensus and establish a curvature-independent disagreement bound for arbitrary step-size schedules. This allows both the optimization and network components of the regret analysis to remain independent of sectional-curvature bounds.

\section{Preliminaries}
\label{sec:background}

We work on a Hadamard manifold $\mathcal{M}$, i.e., a complete, simply connected Riemannian manifold with nonpositive sectional curvature. For $x\in\mathcal{M}$, let $T_x\mathcal{M}$ denote the tangent space, $\langle\cdot,\cdot\rangle_x$ the Riemannian metric with induced norm $\|v\|_x:=\sqrt{\langle v,v\rangle_x}$,
where the subscript is omitted when clear, and $d(\cdot,\cdot)$ the induced geodesic distance. For brevity, we also write $|xy|:=d(x,y)$ for $x,y\in\mathcal{M}$. A set $\mathcal{X}\subseteq\mathcal{M}$ is g-convex if the geodesic segment joining any two points in $\mathcal{X}$ lies entirely in $\mathcal{X}$. The exponential map $\operatorname{Exp}_x:T_x\mathcal{M}\rightarrow\mathcal{M}$ and logarithmic map $\operatorname{Log}_x:\mathcal{M}\rightarrow T_x\mathcal{M}$ are globally well-defined on Hadamard manifolds, see \cite{absil2009optimization,boumal2023introduction} for further background. 

\subsection{Problem Formulation}
\label{subsec:problemform}

\subsubsection{Decentralized Online Riemannian Optimization}
\label{subsubsec:DOR}

Consider $n$ agents communicating over a network described by a symmetric, doubly stochastic matrix $W=(w_{ij})\in\mathbb{R}^{n\times n}$, where $w_{ij}>0$ for $i\neq j$ only if
agents $i$ and $j$ are neighbors \cite{chen2024decentralized,shahrampour2017distributed}. At each round $t\in[T]:=\{1,\ldots,T\}$, agent $i$ selects a decision $x_{i,t}$ from a common feasible set $\mathcal{X}\subseteq\mathcal{M}$ and receives first-order feedback at $x_{i,t}$. The global loss is defined as
\begin{equation}
f_t(x):=\frac{1}{n}\sum_{i=1}^{n} f_{i,t}(x).
\end{equation}
We evaluate the collective performance of the agents through static regret,  which measures the cumulative gap between the agents' decisions and the best fixed collective decision in hindsight, both evaluated under the global loss sequence. Specifically,
\begin{equation}
\label{eq:reg-hconv}
\operatorname{Reg}(T)
:=
\frac{1}{n}\sum_{i=1}^{n}\sum_{t=1}^{T} f_t(x_{i,t})
-
\sum_{t=1}^{T} f_t(x^\ast),
\end{equation}
where $x^\ast\in\arg\min_{x\in\mathcal{X}}\sum_{t=1}^{T} f_t(x)$. The goal is to design a decentralized algorithm that achieves sublinear regret, $\operatorname{Reg}(T)=o(T)$, using only local loss information and neighbor-to-neighbor communication.

To perform consensus on the manifold, we use the weighted Fr\'echet mean as the intrinsic analogue of Euclidean averaging. Given points $\{z_j\}_{j=1}^n\subset\mathcal{M}$ and weights $w_j\geq 0$ with $\sum_{j=1}^n w_j=1$, their weighted Fr\'echet mean is defined as
\begin{equation}
\label{eq:Frechet}
\bar z_w
:=
\arg\min_{z\in\mathcal{M}} \left\{
\sum_{j=1}^n w_j d^2(z,z_j)\right\},
\end{equation}
which is uniquely defined on a Hadamard manifold \cite{sturm2003probability}. In the decentralized setting, agent $i$ uses the weights $\{w_{ij}\}_{j=1}^n$ to aggregate its neighbors' iterates through this minimization, yielding an implicit consensus update \cite{chen2024decentralized}. Unlike linear averaging in Euclidean spaces, the update is defined through an optimization problem over the manifold.

For equal weights, we denote the Fr\'echet mean of a collection $\{z_i\}_{i=1}^n\subset\mathcal{M}$ by $\bar z := \arg\min_{z\in\mathcal{M}} \frac{1}{n}\sum_{i=1}^n d^2(z,z_i)$ and define the corresponding Fr\'echet variance as
\begin{equation}
V_F(\{z_i\}_{i=1}^n)
:=
\frac{1}{n}\sum_{i=1}^n d^2(\bar z,z_i).
\end{equation}
The Fr\'echet variance serves as a geometric measure of disagreement among the agents and will be used to characterize the contraction of the implicit consensus step.

\subsubsection{Horospherical Convexity}
A function $f:\mathcal{M}\rightarrow\mathbb{R}$ is geodesically convex (g-convex) if its restriction to every geodesic is convex in the Euclidean sense \cite{zhang2016first,vishnoi2018geodesic}. H-convexity is a stronger notion on Hadamard manifolds that uses Busemann functions as the analogue of affine functions in Euclidean convexity; we refer to \cite{criscitiello2025horospherically,sahinoglu2025online} for their construction and geometric interpretation. 
For a unit-speed geodesic ray $\gamma:[0,\infty)\to\mathcal M$, its Busemann function is defined as $B_\gamma(x):=\lim_{t\to\infty}\bigl(d(x,\gamma(t))-t\bigr).$ Following \cite{criscitiello2025horospherically,sahinoglu2025online}, for $y\in\mathcal{M}$ and $v\in T_y\mathcal{M}\setminus\{0\}$, let $B_{y,v}:=\|v\|B_{\gamma}$, where $\gamma(0)=y$ and $\dot{\gamma}(0)=-v/\|v\|$; we set $B_{y,0}\equiv 0$ and we recall the following definition.

\begin{definition}[h-convexity]
\label{def:hconvex}
A function $f:\mathcal{M}\rightarrow\mathbb{R}$ is h-convex if, for every $y\in\mathcal{M}$, there exists $v\in T_y\mathcal{M}$ such that
\begin{equation}
f(x)-f(y)\geq B_{y,v}(x),
\qquad \forall x\in\mathcal{M}.
\end{equation}
Such a vector $v$ is called an h-subgradient of $f$ at $y$, and we denote the set of all h-subgradients of $f$ at $y$ as $\partial^hf(y)$.
\end{definition}

Building on h-convexity, we also consider its strongly convex counterpart \cite{criscitiello2025horospherically,sahinoglu2025online}.

\begin{definition}[$\mu$-strong h-convexity]
\label{def:stronghconvex}
A function $f:\mathcal{M}\rightarrow\mathbb{R}$ is $\mu$-strongly h-convex, for $\mu>0$, if for every $y\in\mathcal{M}$, there exists $v\in T_y\mathcal{M}$ such that, for all $x\in\mathcal{M}$,
\begin{equation}
\label{eq:strong-h-conv}
f(x)-f(y)
\geq
Q_{y,v}^{\mu}(x)
:=
-\frac{1}{2\mu}\|v\|^2
+
\frac{\mu}{2}
d^2\!\left(
\operatorname{Exp}_{y}\!\left(-\frac{1}{\mu}v\right),x
\right).
\end{equation}
The set of all such vectors is denoted by $\partial_\mu^h f(y)$. For notational convenience, we set
$\partial_0^h f(y):=\partial^h f(y)$.
\end{definition}

For fixed $y\in\mathcal{M}$ and $v\in T_y\mathcal{M}$, we have $\lim_{\mu\to 0^+}Q_{y,v}^{\mu}(x)=B_{y,v}(x)$ for all $x\in\mathcal{M}$, so the h-convex inequality is recovered in the limit \cite{criscitiello2025horospherically}. Moreover, h-convexity is a subclass of g-convexity; the two notions coincide in Euclidean spaces, while the inclusion can be strict on Hadamard manifolds \cite{criscitiello2025horospherically}. Despite this stronger requirement, the class contains important geometric objectives: in particular, $d(\cdot,z)$ is h-convex and $\frac{1}{2}d^2(\cdot,z)$ is $1$-strongly h-convex, corresponding to geometric median and Fr\'echet mean problems, respectively \cite{criscitiello2025horospherically}.

\subsection{Assumptions}

We impose the following assumptions for the analysis.

\begin{assumption}[Feasible set]
\label{as:min}
The feasible set $\mathcal{X}\subseteq\mathcal{M}$ is compact and g-convex with diameter bounded by $D$, i.e., $d(x,y)\leq D$ for all $x,y\in\mathcal{X}$.
\end{assumption}

The g-convexity of $\mathcal{X}$ on a Hadamard manifold ensures that the metric projection $\mathcal{P}_{\mathcal{X}}(z):=\arg\min_{x\in\mathcal{X}} d(x,z)$ is uniquely defined and nonexpansive, i.e., $d(\mathcal{P}_{\mathcal{X}}(z),\mathcal{P}_{\mathcal{X}}(y))\leq d(z,y)$ \cite{bacak2014convex}. Moreover, since $\mathcal X$ is g-convex, the weighted Fr\'echet mean of points in $\mathcal X$ also belongs to $\mathcal X$ \cite{bacak2014computing}, so minimizing the Fr\'echet objective in Eq. \eqref{eq:Frechet} over $\mathcal M$ or $\mathcal X$ is equivalent.

\begin{assumption}[Network connectivity]
\label{as:network}
The communication matrix $W\in\mathbb{R}^{n\times n}$ is symmetric and doubly stochastic, and satisfies $\sigma_2(W)\in[0,1)$, where $\sigma_2(W)$ denotes the second-largest singular value of $W$.
\end{assumption}

The quantity $1-\sigma_2(W)$ is the spectral gap and quantifies
network connectivity; larger gaps correspond to faster consensus \cite{chen2024decentralized,sahinoglu2025decentralizedonlineriemannianoptimization}.

\begin{assumption}[Local losses]\label{as:losses}
There exists $r>0$ such that for every $i\in[n]$ and $t\in[T]$, the local loss $f_{i,t}$ is $L$-Lipschitz on $\mathcal{X}_r:=\{x\in\mathcal{M}:d(x,\mathcal{X})\le r\}$, i.e.,
$|f_{i,t}(x)-f_{i,t}(y)|\le L\,d(x,y)$ for all $x,y\in\mathcal{X}_r$.
\end{assumption}

Whenever $f_{i,t}$ is h-convex, this assumption bounds every h-subgradient on $\mathcal{X}$. Fix $y\in\mathcal{X}$ and a nonzero $v\in\partial^{h}f_{i,t}(y)$, and set $x_s:=\mathrm{Exp}_y\!\left(s\,v/\|v\|\right)$ with $0<s\le r$, so that $x_s\in\mathcal{X}_r$. Since $B_{y,v}(x_s)=s\|v\|$ by the definition of $B_{y,v}$, Definition~\ref{def:hconvex} and Lipschitzness give
\[
  s\|v\|\;\le\;f_{i,t}(x_s)-f_{i,t}(y)\;\le\;Ls,
\]
hence $\|v\|\le L$. Because $\partial^{h}_{\mu}f_{i,t}(y)\subseteq\partial^{h}f_{i,t}(y)$ \cite[Proposition~2(i)]{criscitiello2025horospherically}, the same bound holds in the strongly h-convex case. The neighborhood condition is needed because the geodesic direction defining an h-subgradient at a boundary point of $\mathcal X$ need not remain in $\mathcal X$. Assumption~\ref{as:losses} therefore controls both the displacement of the local update and the terms arising in the regret analysis.  The convexity requirement is imposed separately in the main results: we consider h-convex losses in Theorem~\ref{thm:hconvex} and $\mu$-strongly h-convex losses in Theorem~\ref{thm:stronghconvex}. Together with the compactness of $\mathcal{X}$, Lipschitz continuity guarantees the existence of a best fixed decision $x^\ast$ defined above.

\subsection{Auxiliary Geometric Lemmas}

We collect here the geometric inequalities used in the regret analysis. Lemmas \ref{lem:threepoint} and \ref{lem:stewart} provide the inequalities underlying the h-convex and strongly h-convex analyses, respectively, while Lemmas \ref{lem:jensen} and \ref{lem:contraction} characterize the averaging and contraction properties of weighted Fréchet-mean consensus.

\begin{lemma}[{\cite[Lemma~2]{criscitiello2025horospherically}}]\label{lem:threepoint}
For any $x,y,p\in\M$, it holds that
\begin{equation}
    \dist^2(p,x)\le \dist^2(p,y)+\dist^2(x,y)+2\,\Bus_{y,-\Log_y(p)}(x).
\end{equation}
\end{lemma}

\begin{lemma}[{\cite[Lemma~2.4]{sahinoglu2025online}}]\label{lem:stewart}
Let $\triangle abc$ be a geodesic triangle and let $p$ lie on the geodesic segment $bc$. Then
\begin{equation}\label{eq:stewart}
    |ab|^2|pc|+|ac|^2|pb|\ge \big(|pa|^2+|pb|\,|pc|\big)|bc|.
\end{equation}
\end{lemma}

\begin{lemma}[{\cite[Theorem~6.2]{sturm2003probability}}]
\label{lem:jensen}
Let $\bar y_w$ be the weighted Fr\'echet mean of $\{y_i\}_{i=1}^n$
with weights $\{w_i\}_{i=1}^n$. For any g-convex function
$f:\mathcal{M}\rightarrow\mathbb{R}$,
\begin{equation}
f(\bar y_w)\leq \sum_{i=1}^n w_i f(y_i).
\end{equation}
\end{lemma}

\begin{lemma}[{\cite[Lemma~3]{chen2024decentralized}}]
\label{lem:contraction}
Let $\{y_i\}_{i=1}^n\subset\mathcal{M}$, and let $x_i$ denote the weighted Fr\'echet mean of $\{y_j\}_{j=1}^n$ with weights $\{w_{ij}\}_{j=1}^n$, where $W=(w_{ij})$ is symmetric and doubly stochastic. Then
\begin{equation}
\label{eq:contraction}
V_F(\{x_i\}_{i=1}^n)
\leq
\sigma_2^2(W)\,
V_F(\{y_i\}_{i=1}^n).
\end{equation}
\end{lemma}

Importantly, this contraction depends only on the communication matrix and requires no sectional-curvature parameter.

\section{Algorithm and Main Results}

\subsection{Distributed Riemannian Online Gradient Descent}

The proposed Distributed Riemannian Online Gradient Descent (D-ROGD) alternates a local Riemannian h-subgradient update with an implicit weighted Fr\'echet-mean consensus step. At round $t$, agent $i$ computes 
\begin{equation*}
    \nabla_{i,t}\in\begin{cases}
        \partial^hf_{i,t}(x_{i,t}),&\text{when }f_{i,t} \text{ is h-convex},\\
        \partial^h_{\mu}f_{i,t}(x_{i,t}),&\text{when }f_{i,t} \text{ is $\mu$-strongly h-convex},
    \end{cases}
\end{equation*}
takes a projected h-subgradient step to obtain $y_{i,t+1}$, and then computes $x_{i,t+1}$ by weighted Fr\'echet-mean consensus, as summarized in Algorithm~\ref{alg:drogd}.

\begin{algorithm}[t]
\caption{Distributed Riemannian Online Gradient Descent (D-ROGD)}
\label{alg:drogd}
\begin{algorithmic}[1]
\STATE \textbf{Input:} feasible set $\X\subseteq\M$, horizon $T$, step sizes $\{\eta_t\}$, matrix $W$, common initial point $x_1\in\mathcal X$, where $x_{i,1}=x_1$ for all $i\in[n]$.
\FOR{$t=1$ to $T$}
    \FOR{each agent $i=1,\dots,n$ (in parallel)}
        \STATE Observe $\nabla_{i,t}\in\partial^h_{\mu}f_{i,t}(x_{i,t})$ ($\mu=0$ in the h-convex case).
        \STATE $y_{i,t+1}=P_{\X}\big(\Exp_{x_{i,t}}(-\eta_t\nabla_{i,t})\big)$.
        \STATE $x_{i,t+1}=\argmin_{z\in\X}\left\{\sum_{j=1}^n w_{ij}\,\dist^2(y_{j,t+1},z)\right\}$.
    \ENDFOR
\ENDFOR
\end{algorithmic}
\end{algorithm}
 
For the analysis, let $\bar x_t$ denote the Fr\'echet
mean of $\{x_{i,t}\}_{i=1}^n$. We decompose the static regret as
\begin{align}
\Reg(T)
&= \underbrace{\frac{1}{n}\sum_{i=1}^n\sum_{t=1}^T\big(f_t(x_{i,t})-f_{i,t}(x_{i,t})\big)}_{\text{(I): network error}} \nonumber\\
&\quad + \underbrace{\frac{1}{n}\sum_{i=1}^n\sum_{t=1}^T f_{i,t}(x_{i,t})-\sum_{t=1}^T f_t(x^\ast)}_{\text{(II): optimization error}}. \label{eq:decomp}
\end{align}
Term~(I) captures the network disagreement, whereas term~(II) captures the local optimization error.

\subsection{Curvature-Independent Network Error}
The following lemma bounds the disagreement among the agents under arbitrary step size schedules in terms of the network spectral gap. This curvature-independent estimate is the main tool for controlling the network-error term in~\eqref{eq:decomp}.

\begin{lemma}[Network Error]\label{lem:neterror}
Let Assumptions \ref{as:min}, \ref{as:network} and \ref{as:losses} hold. Then, for any nonnegative step size sequence $\eta_t$, the iterates of Algorithm~\ref{alg:drogd} satisfy
\begin{equation}\label{eq:neterror}
    \frac{1}{n}\sum_{t=1}^T\sum_{i=1}^n \dist(x_{i,t},\bar{x}_t)\le \frac{L\sigmatwo}{1-\sigmatwo}\sum_{k=1}^{T-1}\eta_k.
\end{equation}
\end{lemma}

\begin{proof}
Let $\sigma:=\sigma_2(W)$ and define $a_t:=V_F(\{x_{i,t}\}_{i=1}^n)^{1/2}$. By Cauchy--Schwarz, $ \frac{1}{n}\sum_{i=1}^n d(x_{i,t},\bar{x}_t)\leq a_t.$
Applying Lemma~\ref{lem:contraction} and using the definition of the Fr\'echet variance,
\begin{equation}
a_t
\leq
\sigma V_F(\{y_{i,t}\}_{i=1}^n)^{1/2}
\leq
\sigma
\left(
\frac{1}{n}\sum_{i=1}^n d^2(\bar{x}_{t-1},y_{i,t})
\right)^{1/2}.    
\end{equation}

By triangle inequality on $\mathcal{M}$ and Minkowski's inequality,
$$
\left(
\frac{1}{n}\sum_{i=1}^n
d^2(\bar{x}_{t-1},y_{i,t})
\right)^{1/2}
\leq
a_{t-1}
+
\left(
\frac{1}{n}\sum_{i=1}^n
d^2(x_{i,t-1},y_{i,t})
\right)^{1/2}.
$$
Moreover, by the nonexpansiveness of $\mathcal{P}_{\mathcal{X}}$ and
the bound $\|\nabla_{i,t-1}\|\leq L$,
we have $d(x_{i,t-1},y_{i,t})\leq \eta_{t-1}\|\nabla_{i,t-1}\| \leq L\eta_{t-1}.$
Therefore, $a_t\leq \sigma a_{t-1}+\sigma L\eta_{t-1}$.

Since the agents are initialized at a common point, $a_1=0$. Unrolling the recursion gives $a_t \leq L\sum_{k=1}^{t-1}\sigma^{t-k}\eta_k.$ Therefore,
\begin{equation}
\frac{1}{n}\sum_{t=1}^{T}\sum_{i=1}^{n}
d(x_{i,t},\bar{x}_t)
\leq
\sum_{t=1}^{T}a_t
\leq
L\sum_{k=1}^{T-1}\eta_k
\sum_{t=k+1}^{T}\sigma^{t-k}.    
\end{equation}

Using $\sum_{t=k+1}^{T}\sigma^{t-k}\leq \sigma/(1-\sigma)$ for $\sigma\in[0,1)$ yields
\begin{equation}
\frac{1}{n}\sum_{t=1}^{T}\sum_{i=1}^{n}
d(x_{i,t},\bar{x}_t)
\leq
\frac{L\sigma}{1-\sigma}
\sum_{k=1}^{T-1}\eta_k,    
\end{equation} which proves the claim.
\end{proof}

To bound the network-error term, add and subtract
$f_t(\bar{x}_t)$ and $f_{i,t}(\bar{x}_t)$ and obtain $f_t(x_{i,t})-f_{i,t} (x_{i,t}) = f_t(x_{i,t})-f_t(\bar{x}_t) + f_t(\bar{x}_t)-f_{i,t}(\bar{x}_t) + f_{i,t}(\bar{x}_t)-f_{i,t}(x_{i,t}).$
Averaging over $i$, the term $\sum_i\left(f_t(\bar{x}_t)-f_{i,t}(\bar{x}_t)\right)$ vanishes since
$f_t=\frac{1}{n}\sum_{i=1}^n f_{i,t}$. Moreover, both $f_t$ and
$f_{i,t}$ are $L$-Lipschitz, so
$
\text{(I)}
\leq
\frac{2L}{n}
\sum_{t=1}^{T}\sum_{i=1}^{n}
d(x_{i,t},\bar{x}_t).
$
Applying Lemma~\ref{lem:neterror} yields
\begin{equation}
\label{eq:term1}
\text{(I)}
\leq
\frac{2L^2\sigma_2(W)}{1-\sigma_2(W)}
\sum_{t=1}^{T-1}\eta_t.
\end{equation}

\subsection{Main Results}
We now combine the network-error bound in~\eqref{eq:term1} with bounds on the optimization-error term~(II) to establish the regret guarantees for h-convex and strongly h-convex objectives.

\begin{theorem}[Regret for h-convex objectives]\label{thm:hconvex}
Let Assumptions~\ref{as:min}--\ref{as:losses} hold and suppose $f_{i,t}$ is h-convex for every $i\in[n]$, $t\in[T]$. Then Algorithm~\ref{alg:drogd} with constant step size $\eta_t= \eta =1/\sqrt{T}$ satisfies
\begin{equation}\label{eq:hconvexreg}
\Reg(T)\le \Big(\frac{2L^2\,\sigmatwo}{1-\sigmatwo}+\frac{L^2+D^2}{2}\Big)\,\sqrt{T}.
\end{equation}
\end{theorem}

\begin{proof}
Let $z_{i,t}:=\operatorname{Exp}_{x_{i,t}}(-\eta\nabla_{i,t})$. By Definition \ref{def:hconvex} with $v=\nabla_{i,t}$, $f_{i,t}(x_{i,t})-f_{i,t}(x^\ast) \leq -B_{x_{i,t},\nabla_{i,t}}(x^\ast)$. Since $-\operatorname{Log}_{x_{i,t}}(z_{i,t})=\eta\nabla_{i,t}$, the positive homogeneity of the scaled Busemann function gives $-B_{x_{i,t},\nabla_{i,t}}(x^\ast) = -\frac{1}{\eta} B_{x_{i,t}, -\operatorname{Log}_{x_{i,t}}(z_{i,t})}(x^\ast)$. Applying Lemma~\ref{lem:threepoint} with $y=x_{i,t}$, $p=z_{i,t}$, and $x=x^\ast$ yields
\begin{equation}
\label{eq:fdiff}
f_{i,t}(x_{i,t})-f_{i,t}(x^\ast) \leq \frac{1}{2\eta} \left[ d^2(z_{i,t},x_{i,t}) +d^2(x^\ast,x_{i,t}) -d^2(z_{i,t},x^\ast) \right].
\end{equation}

Since $d(z_{i,t},x_{i,t})=\eta\|\nabla_{i,t}\|\leq \eta L$, and $x^\ast\in\mathcal{X}$, the nonexpansiveness of $\mathcal{P}_{\mathcal{X}}$ gives $d(y_{i,t+1},x^\ast) = d\!\left( \mathcal{P}_{\mathcal{X}}(z_{i,t}), \mathcal{P}_{\mathcal{X}}(x^\ast) \right) \leq d(z_{i,t},x^\ast)$.
Hence,
\begin{align}\label{eq:19}
f_{i,t}(x_{i,t})-f_{i,t}(x^\ast)
\leq{}&
\frac{\eta L^2}{2}
+
\frac{1}{2\eta}
\left[
d^2(x^\ast,x_{i,t})
-d^2(x^\ast,x_{i,t+1})
\right]
\nonumber\\
&+
\frac{1}{2\eta}
\left[
d^2(x^\ast,x_{i,t+1})
-d^2(x^\ast,y_{i,t+1})
\right].
\end{align}

Since $d^2(\cdot,x^\ast)$ is g-convex on a Hadamard manifold, Lemma~\ref{lem:jensen} applied to the weighted Fr\'echet-mean consensus step gives $d^2(x^\ast,x_{i,t+1}) \leq \sum_{j=1}^n w_{ij}d^2(x^\ast,y_{j,t+1})$. Summing over $i$ and using the double stochasticity of $W$ yields
\begin{equation}
\label{eq:consgap}
\sum_{i=1}^n d^2(x^\ast,x_{i,t+1})
\leq
\sum_{i=1}^n d^2(x^\ast,y_{i,t+1}).
\end{equation}
Therefore, by averaging \eqref{eq:19} over $i$ and using~\eqref{eq:consgap}, the last term in \eqref{eq:19} vanishes. Also, another term telescopes over $t$ and we get
\begin{align}
\text{(II)} &\leq \frac{\eta L^2}{2}T + \frac{1}{2\eta n} \sum_{i=1}^n \left[ d^2(x^\ast,x_{i,1}) - d^2(x^\ast,x_{i,T+1}) \right] \\
&\leq \frac{\eta L^2}{2}T+\frac{D^2}{2\eta}.    
\end{align}
Combining this bound with~\eqref{eq:term1} and setting $\eta=1/\sqrt{T}$ yields the result.
\end{proof}

\begin{theorem}[Regret for strongly h-convex objectives]\label{thm:stronghconvex}
Let Assumptions~\ref{as:min}--\ref{as:losses} hold and suppose $f_{i,t}$ is $\mu$-strongly h-convex for every $i\in[n]$, $t\in[T]$. Then Algorithm~\ref{alg:drogd} with the diminishing step-size schedule $\eta_t=1/(\mu t)$ satisfies
\begin{equation}
\operatorname{Reg}(T)
\leq
\frac{1}{\mu}
\left(
\frac{2L^2\sigma_2(W)}{1-\sigma_2(W)}
+
\frac{L^2}{2}
\right)
(1+\log T).
\end{equation}
\end{theorem}
 
\begin{proof}
Let $u_{i,t}:=\Exp_{x_{i,t}}(-\mu^{-1}\nabla_{i,t})$ and
$z_{i,t}:=\Exp_{x_{i,t}}(-\eta_t\nabla_{i,t})$. By $\mu$-strong
h-convexity and Definition~\ref{def:stronghconvex},
\[
  f_{i,t}(x_{i,t})-f_{i,t}(x^*)\;\le\;\tfrac{1}{2\mu}\|\nabla_{i,t}\|^2
  -\tfrac{\mu}{2}\,d^2(u_{i,t},x^*).
\]
Since $\eta_t=1/(\mu t)\leq1/\mu$, the point $z_{i,t}$ lies on the geodesic segment
from $x_{i,t}$ to $u_{i,t}$, with
$|u_{i,t}x_{i,t}|=\mu^{-1}\|\nabla_{i,t}\|$,
$|z_{i,t}x_{i,t}|=\eta_t\|\nabla_{i,t}\|$ and
$|z_{i,t}u_{i,t}|=(\mu^{-1}-\eta_t)\|\nabla_{i,t}\|$. Applying
Lemma~\ref{lem:stewart} to $\triangle x^*x_{i,t}u_{i,t}$ with
$p=z_{i,t}$ and simplifying gives
\begin{equation*}
    \begin{aligned}
        d^2(u_{i,t},x^*)\ge&\tfrac{1}{\mu\eta_t}d^2(z_{i,t},x^*)
        +\tfrac{1}{\mu}\Big(\tfrac{1}{\mu}-\eta_t\Big)\|\nabla_{i,t}\|^2\\
        &-\tfrac{1}{\eta_t}\Big(\tfrac{1}{\mu}-\eta_t\Big)d^2(x^*,x_{i,t}).
    \end{aligned}
\end{equation*}
Substituting this bound and using $d(z_{i,t},x^*)\ge d(y_{i,t+1},x^*)$
by non-expansiveness, we obtain
\begin{align*}
\textstyle f_{i,t}(x_{i,t})-f_{i,t}(x^\ast)
\textstyle\le& \textstyle\frac{\eta_t}{2}\cdot \|\nabla_{i,t}\|^2 - \textstyle\frac{1}{2\eta_t}\cdot \bigl|\,y_{i,t+1}\,x^* \bigr|^2\\
&+\textstyle\left(\frac{1}{2\eta_t}-\frac{\mu}{2}\right)\cdot \bigl|\,x^*\,x_{i,t}\bigr|^2\\
\textstyle=& \textstyle\frac{\eta_t}{2}\cdot \|\nabla_{i,t}\|^2 + \frac{1}{2\eta_t}\cdot \textstyle\bigl(\bigl|\,x_{i,t+1}\,x^* \bigr|^2 - \bigl|\,y_{i,t+1}\,x^* \bigr|^2\bigr)\\
    &\textstyle+ \left(\frac{1}{2\eta_t} - \frac{\mu}{2}\right)\cdot \bigl|\,x^*\,x_{i,t}\bigr|^2 - \frac{1}{2\eta_t}\cdot \bigl|\,x_{i,t+1}\,x^* \bigr|^2
\end{align*}
Averaging over $i$, applying \eqref{eq:consgap} again, and using $\eta_t=1/(\mu t)$, for which $\tfrac{1}{2\eta_t}-\tfrac{1}{2\eta_{t-1}}-\tfrac\mu2=0$, $T\ge t\ge2$, the distance terms telescope and yield
\begin{equation}\label{eq:term2strong}
\text{(II)}\le \frac{L^2}{2}\sum_{t=1}^{T}\eta_t
\le \frac{L^2}{2\mu}\big(\log T+1\big).
\end{equation}
Combining \eqref{eq:term1} and \eqref{eq:term2strong} with $\eta_t=1/(\mu t)$ and $\sum_{t=1}^{T}1/(\mu t)\le \tfrac{1}{\mu}(1+\log T)$ gives the stated bound.
\end{proof}

\begin{remark}\label{rem:curv}
The regret bounds are curvature-independent and match the Euclidean dependence on $T$, while decentralization enters only through the spectral-gap factor $\sigma_2(W)/(1-\sigma_2(W))$. In particular, when $\sigma_2(W)=0$, the network-error term vanishes, recovering the corresponding single-agent regret order. This improves upon prior distributed g-convex analyses, where curvature-dependent quantities enter the regret bounds or the consensus step size \cite{chen2024decentralized, sahinoglu2025decentralizedonlineriemannianoptimization}, though this improvement is only for the narrower class of h-convex functions.
\end{remark}

\section{Numerical Experiments}
\label{sec:numerical}

We evaluate D-ROGD on real hierarchical data embedded in hyperbolic space, varying manifold curvature and network connectivity. Fixing the intrinsic data diameter allows us to assess curvature dependence at a common geometric scale~\footnote{The code is available at: \href{https://github.com/ZhanyCai-opt/decentralized-online-riemannian-curvature-independent}{https://github.com/ZhanyCai-opt/decentralized-online-riemannian-curvature-independent}.}.

\textbf{Data and manifold.}
From the WordNet noun taxonomy \cite{miller1995wordnet} we take the sub-hierarchy rooted at \texttt{mammal.n.01}, retain $N=300$ synsets, and embed the tree's all-pairs shortest-path distances into the $10$-dimensional Poincar\'e ball \cite{nickel2017poincare} by stress minimization. For each $\kappa\in\{0.25,1,4,9,16\}$ the embedding is rescaled radially in $\mathbb{H}^{10}_{\kappa}$, the Poincar\'e ball of curvature $-\kappa$, so that the intrinsic geodesic diameter remains fixed at $D=4$ while the curvature varies. The resulting curvature--diameter parameter is $\sqrt{\kappa}D\in \{2,4,8,12,16\}$, which appears in the comparison constants of g-convex analyses \cite{zhang2016first}.

\textbf{Network and protocol.}
We use $n=8$ agents on a ring with Metropolis weights, each assigned a distinct semantic cluster of concepts by $k$-means on the embedding, so that local objectives are heterogeneous. At round $t$ agent $i$ draws a concept uniformly from its cluster and observes a perturbed copy $S_{i,t}$. We run $T=2000$ rounds over three seeds. The feasible set $\mathcal{X}$ is a geodesic ball containing the data;  the consensus step, i.e., line 6 of Algorithm~\ref{alg:drogd}, is solved by Karcher iteration to tolerance $10^{-9}$ \cite{bacak2014computing}, and $x^{\ast}$ by projected Riemannian gradient descent over all $nT$ observations in $\mathcal{X}$.

\textbf{Objectives.}
The h-convex case is the online distributed geometric median, $f_{i,t}(x)=d(x,S_{i,t})$ with $\eta_t=1/\sqrt{t}$ which mimics the $1/\sqrt{T}$ scaling of Theorem~\ref{thm:hconvex} without requiring prior knowledge of the horizon; the strongly h-convex case is the online distributed Fr\'echet mean, $f_{i,t}(x)=\tfrac12 d^2(x,S_{i,t})$ with $\eta_t=1/(\mu t)$, $\mu=1$. For the geometric median, we use the canonical h-subgradient $\nabla_{i,t}=-\mathrm{Log}_{x_{i,t}}(S_{i,t})/d(x_{i,t},S_{i,t})$ when $x_{i,t}\neq S_{i,t}$ and zero otherwise.

\textbf{Results.} Fig.~\ref{fig:regret} shows that the cumulative regret grows consistently with the predicted $O(\sqrt{T})$ and $O(\log T)$ rates for the h-convex and strongly h-convex objectives, respectively, across all tested curvatures. Fig.~\ref{fig:curv-net}(a) plots the normalized regret against $\sqrt{\kappa}D$. While the g-convex comparison factor $\zeta(\kappa,D)=\sqrt{\kappa}D\coth(\sqrt{\kappa}D)$ grows substantially with curvature, the normalized regret exhibits only mild variation in the h-convex case and decreases in the strongly h-convex case, showing no systematic curvature-dependent deterioration comparable to that induced by $\zeta(\kappa,D)$. Fig. ~\ref{fig:curv-net}(b) varies the network topology and shows increasing regret as $\sigma_2(W)$ approaches one, consistent with the spectral-gap dependence $\sigma_2(W)/(1-\sigma_2(W))$ in the theoretical bounds. Taken together, these results are consistent with the theoretical separation between curvature-independent optimization guarantees and spectral-gap-dependent network effects.


\begin{figure}
  \centering
  \includegraphics[width=\linewidth]{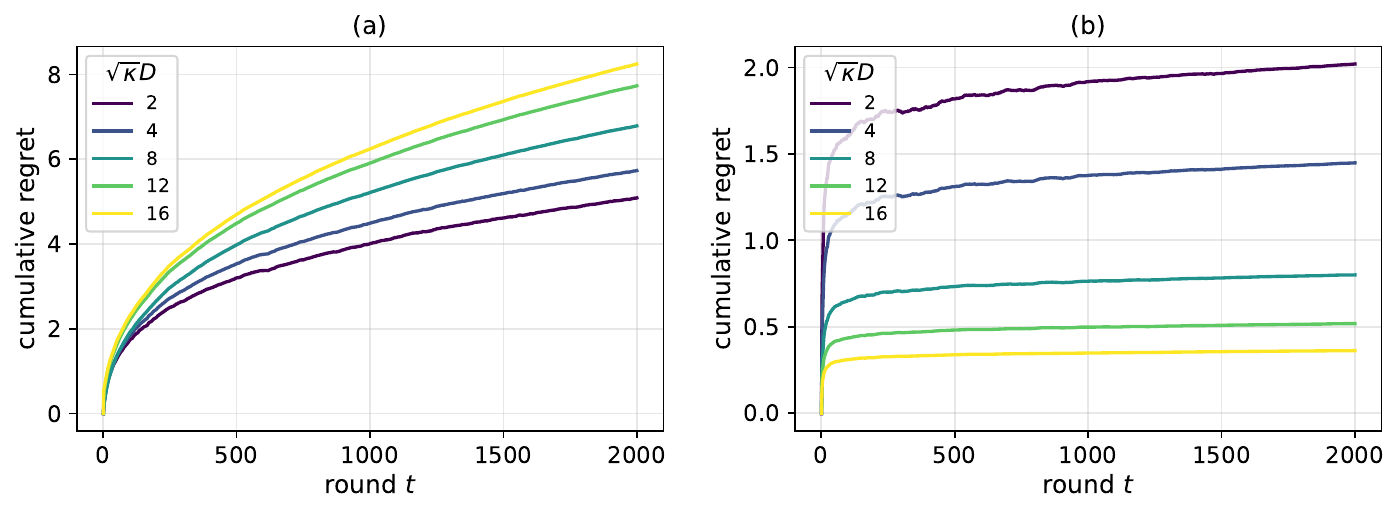}
  \caption{Cumulative network regret of D-ROGD on the embedded WordNet hierarchy data, with $D=4$ and $n=8$. Left: h-convex objective (geometric median). Right: strongly h-convex objective (Fr\'echet mean).}
  \label{fig:regret}
\end{figure}

\begin{figure}
  \centering
  \includegraphics[width=\linewidth]{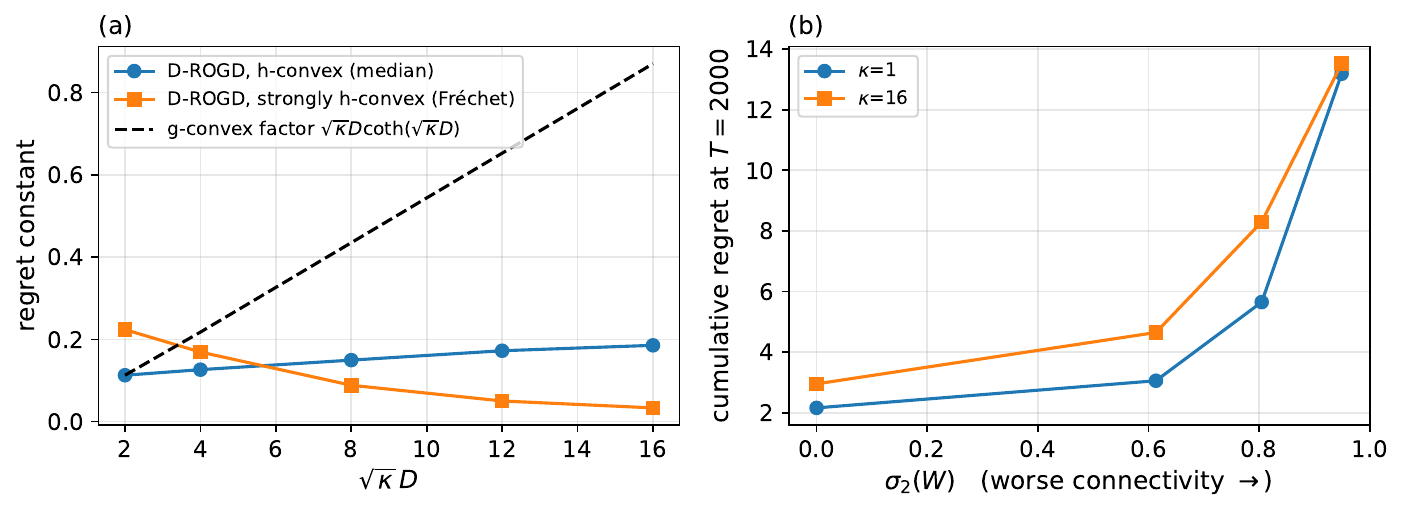}
  \caption{(a) Regret constants at $T=2000$ against $\sqrt{\kappa}D$: $\mathrm{Reg}(T)/\sqrt{T}$ (h-convex) and $\mathrm{Reg}(T)/\log T$ (strongly h-convex). The dashed curve is the g-convex curvature factor $\zeta(\kappa,D)=\sqrt{\kappa}D\coth(\sqrt{\kappa}D)$, scaled to start at the leftmost h-convex point so that slopes, not heights, are comparable.
  (b) Regret of the h-convex objective on complete, Erd\H{o}s--R\'enyi
  ($p=0.5$), ring and path graphs at two curvatures.}
  \label{fig:curv-net}
\end{figure}
 
\section{Conclusions}\label{sec:conclusion}
We studied distributed online Riemannian optimization on Hadamard manifolds beyond the standard g-convexity framework, analyzing h-convex and strongly h-convex objectives. We proposed D-ROGD, which couples a local Riemannian h-subgradient step with a Fr\'echet-mean consensus step, and established curvature-independent regret bounds of $O(\sqrt{T})$ and $O(\log T)$ for h-convex and strongly h-convex objectives, respectively. The bounds match the Euclidean rates in $T$ and depend on the network only through the spectral-gap factor $\sigmatwo/(1-\sigmatwo)$, thereby overcoming the curvature-dependent constants that arise in decentralized online g-convex analyses. Future directions include dynamic regret and time-varying networks, bandit feedback, and applications in distributed control and large-scale learning on curved spaces.

\addtolength{\textheight}{-12cm}   






\bibliographystyle{IEEEtran}
\bibliography{refs}

\end{document}